\pdfoutput=1
\documentclass[nohyperref]{article}

\usepackage[numbers]{natbib}

\usepackage{microtype}
\usepackage{graphicx}
\usepackage{subfigure}
\usepackage{bm}
\usepackage{paralist}
\usepackage{booktabs} 

\usepackage[accepted]{icml2023}

\usepackage[accepted]{icml2023}

\usepackage{amsthm}
\usepackage{mathtools}
\usepackage{thmtools,thm-restate}
\usepackage{amssymb}
\usepackage[capitalize,noabbrev]{cleveref}

\theoremstyle{plain}
\newtheorem{theorem}{Theorem}[section]
\newtheorem{proposition}[theorem]{Proposition}
\newtheorem{lemma}[theorem]{Lemma}

\theoremstyle{definition}
\newtheorem{definition}[theorem]{Definition}

\theoremstyle{remark}

\usepackage[textsize=tiny]{todonotes}

\icmltitlerunning{Nonparametric Density Estimation under Distribution Drift}

\DeclareMathOperator*{\Exp}{\mathbb{E}}
\DeclareMathOperator*{\Var}{\mathbb{V}}
\DeclareMathOperator*{\argmax}{\mbox{argmax}}
\DeclareMathOperator*{\TV}{TV}
\DeclareMathOperator*{\KL}{KL}

\begin{document}

\twocolumn[
\icmltitle{Nonparametric Density Estimation under Distribution Drift }



\icmlsetsymbol{equal}{*}

\begin{icmlauthorlist}
\icmlauthor{Alessio Mazzetto}{brown}
\icmlauthor{Eli Upfal}{brown}
\end{icmlauthorlist}

\icmlaffiliation{brown}{Brown University}

\icmlcorrespondingauthor{Alessio Mazzetto}{alessio\_mazzetto@brown.edu}

\icmlkeywords{Machine Learning, ICML, Density Estimation, Drift, Non-Stationary, Minimax, Lower Bound}

\vskip 0.3in
]



\printAffiliationsAndNotice{}  

\begin{abstract}
We study nonparametric density estimation in non-stationary drift settings. Given a sequence of independent samples taken from a distribution that gradually changes in time, the goal is to compute the best estimate for the current distribution. 
We prove tight minimax risk bounds for both discrete and continuous smooth densities, where the minimum is over all possible estimates and the maximum is over all possible distributions that satisfy the drift constraints. Our technique handles a broad class of drift models and generalizes previous results on agnostic learning under drift.
\end{abstract}

\section{Introduction}

Density estimation is a fundamental concept in statistics with numerous applications in data analysis and machine learning. Given a set of samples, the goal is to best estimate the probability distribution that generated these samples, often subject to some parametric or nonparametric assumptions on the family of candidate distributions. This problem has been studied extensively, for both discrete and continuous distribution functions, assuming that the samples are independent and identically distributed according to the distribution that we aim to estimate~\cite{devroye2001combinatorial}. 

In many data analysis applications, ranging from customers' preferences to weather conditions, the assumption that the samples are identically distributed is unrealistic. The underlying distribution is gradually changing over time, and estimating the current distribution inevitably relies on past data from related but not identical distributions. This work presents the first tight bounds for density estimates of both discrete and continuous smooth distributions from samples that are independent but are generated by distributions that drift in time.

Distribution drift has been studied in the context of agnostic learning with the assumption of equal bounded drift at each step \cite{bartlett1992learning}, i.e. 
there is a constant $\Delta > 0$, such that the drift between two distribution $i$ steps apart is bounded by $i\Delta$.
In this case, it has been shown that the minimax risk for learning a family of binary classifiers with VC dimension $\nu$ is $\Theta( (\nu \Delta)^{1/3})$ \cite{barve1996complexity,long1999complexity}, where the minimax risk characterizes the maximal expected error of the best algorithm that solves the problem within the specified drift assumptions. The minimax risk has not been studied with other drift patterns in this context.

In this work, we study the more general problem of density estimation, under a more detailed family of drift models. In particular, our results apply to any \textit{regular} drift sequence, where the $i$th element of this sequence provides an upper bound to the distance of the current distribution and the $i$ distribution in the sequence, and the regularity assumption prevents any abrupt change in the drift (\Cref{def:regular-sequence}). The distance used depends on the specific estimation problem, in particular, we will use the total variation distance in the case of discrete densities, and the $L_2$ distance in the case of smooth densities. The bounded drift per step is one possible model in our setting. However, when more information is available, we obtain more informative tight bounds.

For the problem of estimating a discrete distribution with support size $k$ using $n$ samples from a drifting distribution, we show a minimax risk of $\Theta \left( \sqrt{k/r}\right)$ with respect to the total variation distance, where $r \leq n$ is an integer that is easily derived from the drift sequence. For the special case of no drift, we retrieve the known minimax risk $\Theta( \sqrt{k/n})$ for the estimation of a discrete density with $n$ independent and identically distributed samples. Since the problem of estimating a discrete density with support size $k$ with respect to the total variation distance can be reduced to the problem of agnostic learning a family of binary classifiers with VC dimension $k$, we also show that our results imply a lower bound on the minimax risk for the latter problem. In particular, our results generalize the previously known lower bound that only holds for the case of bounded drift at each step \cite{barve1996complexity}.

The following simple example demonstrates the power of our approach.  Assume a sample space of size $k$, and a distribution drift that follows this pattern: a probability mass of $1-\Delta$ is distributed between the $k$ elements and does not change in time. The remaining $\Delta$ probability mass is redistributed between the $k$ elements at each step. The drift in each step is bounded by $\Delta$ and based only on this bound the best estimate can only guarantee a $\Theta((k\Delta)^{1/3})$ error. However, the drift between the current distribution and any past distribution is also bounded by $\Delta$. Incorporating this extra information, our technique provides a tighter $\Theta (\Delta)$ error estimate. 
In \Cref{sub:agnostic-learning}, we show a similar gap for the agnostic learning problem.

For the smooth density estimation problem, we seek to estimate a probability density that is $\beta$-smooth (\Cref{def:smooth}).
We
focus on the nonparametric case, i.e. we do not assume
any parametric assumption on the target density. 
We establish a minimax risk of $\Theta\left( r^{-2\beta/(2\beta+1)}
\right)$ with respect to the integrated squared loss, where again $r \leq n$ depends on the drift sequence. This results extends the known minimax risk $\Theta\left( n^{-2\beta/(2\beta+1)}
\right)$ for estimating a density from $n$ independent and identically distributed samples \cite{van2000asymptotic}. 

The results we have discussed so far establish the minimax risk for learning the current distribution at a given specific time based on past data. However, we are also interested in characterizing the minimax rate of the average risk for the online version of those problems, where we are required to provide an estimate of the current distribution at each step. This is a more challenging problem, as in the lower bound construction, we need to show that we frequently incur a high estimation error.
Nonetheless, we show that in the case of a bounded drift $\Delta$ at each step, the minimax rate for the online estimation of a discrete density with support size $k$ is $\Theta( (k \Delta)^{1/3})$. As previously discussed, this result also applies to the problem of agnostic learning a family of binary classifiers. This is the first work in the literature to provide a characterization of the minimax rate for the online version of these problems in a distribution drift setting.

Following previous work on this setting, our upper bounds on the minimax risk are obtained by considering an estimator over a window of recent samples of a properly chosen size. The size of the window is chosen to minimize the trade-off between the variance of the estimator and the error introduced by considering samples from distributions that are further away in time and exhibit a large drift. In the literature, the only lower bound construction for drifting distribution is specific to the problem of agnostic learning of binary classifiers \cite{barve1996complexity}, and it assumes a bounded drift at each step. In our paper, we develop a novel proof  strategy that allows us to obtain tight lower bounds for both the problems of discrete density estimation, and smooth density estimation under any arbitrary regular drift sequence.   We believe that our method is of independent interest, and can be possibly applied to other estimation problems with drifting distributions.

\subsection{Related Work}
The distribution drift setting has been introduced in the context of the supervised learning of binary predictors \cite{helmbold1991tracking,bartlett1992learning,helmbold1994tracking}. 
In this line of work, it has been shown that there exists an algorithm that finds a binary predictor whose expected prediction error with respect to the current distribution is at most $O(\sqrt[3]{\nu\Delta})$ larger than the expected error of the best predictor in the family, where $\nu$ is the VC-dimension of the considered family of binary predictors and $\Delta$ is an upper bound to the total variation distance of two consecutive distributions \cite{long1999complexity}. This upper bound is tight \cite{barve1996complexity}. More recent work generalized this analysis to provide upper bounds to learning any family of predictors with bounded Rademacher complexity, and introduced a finer measure of distance between consecutive distributions \cite{mohri2012new}: it uses tools from transfer learning theory \cite{mansour2009domain}, as we can observe that learning with distribution drift is a special case of learning with domain shift \cite{ben2010theory}.

The problem of density estimation in the case of independent and identically distributed samples has been extensively studied in the literature, see \cite{DBLP:books/sp/Silverman86,groeneboom2014nonparametric,scott2015multivariate} for an overview of old and recent work in this topic.
In the case of estimating a distribution with finite support size $k$, it is folklore that we can achieve an expected error $O(\sqrt{k/n})$ in total variation distance with $n$ samples. This upper bound is tight \cite{anthony1999neural}, and the minimax risk bound has been computed with exact constants \cite{kamath2015learning}. Recent work provided an analysis for the estimation of a discrete distribution with infinite countable support \cite{berend2013sharp,han2015minimax}, also using data-dependent bounds \cite{cohen2020learning}.
In the case of the estimation of a $\beta$-smooth density from $n$ independent and identically distributed samples, it is possible to obtain an expected squared error of $O\left(n^{-\frac{2\beta}{2\beta+1}}\right)$, and we refer to  \cite{tsybakov2009introduction} for a recent book on the topic. This upper bound is achieved by using different methods as kernel density estimation (e.g., see \citet{van2000asymptotic}), and it can be proven to be tight by using information-theoretic methods from minimax theory \cite{Devroye1987NonparametricDE,yu1997assouad}.

The problem of relaxing the independent and identically distributed assumptions on the samples for density estimation has been studied in the literature from a theoretical perspective. However, these work significantly diverge from our setting as they use different sets of assumptions. Multiple work addressed the problem of estimating the stationary distribution of a Markov process \cite{roussas1969nonparametric, wen2020batch}, even for arbitrary initial distribution \cite{gillert1984density}. In \cite{phillips1998nonstationary}, the authors developed an asymptotic theory for the kernel density estimate of a random walk and the kernel regression estimator of a non-stationary ﬁrst order autoregression. More similar to our setting, recent work \cite{gokcesu2017online} provides parametric density estimation results for a family of exponential distributions where the parameters of the distributions are allowed to slowly change at each step. Many other algorithms have also been proposed for online learning of densities \cite{kristan2011multivariate,garcia2012online}, however, they do not come with any theoretical analysis.

The problem of characterizing the minimax rate of the average risk for the online version of the density estimation problems can also be related to recent work on online forecasting \cite{baby2019online,baby2020adaptive}. In the online forecasting problem, the goal is to predict the current position of a vector that moves over time, given a sequence of independent noisy observations of the past positions of this vector. For this problem, they provide an adaptive algorithm that achieves an optimal minimax regret (up-to-logarithmic factors) with respect to the total variation of the position of the vector over time. It is  possible to use their algorithm to adaptively estimate a discrete density over $[k]$ in a non-stationary setting. In fact, we can describe each discrete density as a random vector over $\mathbb{R}^k$, where coordinate $i$ represents the probability of sampling the $i$-th element, and we observe a noisy observation of this vector. However, their approach estimates each coordinate independently, thus a trivial application of their work does not achieve the optimal dependency on $k$ for the problem of discrete density estimation with a bounded drift $\Delta$ at each step. 

\subsection{Our Contributions}
\begin{compactenum}
\item We introduce the concept of \emph{ regular drift sequence} - a general framework for characterizing distribution drift (\Cref{sec:distribution-drift}).
\item We establish the minimax risk for discrete density estimation with respect to any regular drift sequence (\Cref{sec:discrete}).
\item We show a generalization of previous lower bound for the problem of agnostic learning a family of binary classifier to any regular drift sequence (\Cref{sub:agnostic-learning}).
\item We establish the minimax rate for the online problem of estimating a discrete density with a bounded drift at each step (\Cref{sub:minimax-discrete-avg}).
\item We establish the minimax risk for estimating a smooth density with respect to any regular drift sequence (\Cref{sec:smooth})
\end{compactenum}

\section{Preliminary}
Let $[n] = \{1,\ldots,n\}$ for $n \in \mathbb{N}$. Consider a non-empty \emph{sample space} $\mathcal{X}$ equipped with a $\sigma$-algebra. Let $(X_i)_{i \in \mathbb{N}}$ be a an independent\footnote{A stochastic process is independent iff every finite subset of its random variables is mutually independent.} stochastic process defined over $\mathcal{X}$, and let $P_i$ be the probability distribution of the random variable $X_i$.
Given $n \in \mathbb{N}$, let $\bm{X}_n = (X_1,\ldots,X_n)$ be the random vector of the first $n$ elements of the random process. Since the $X_i$'s are independent, the distribution of $\bm{X}_n$, can be written as a \emph{product distribution} $S = P_1 \times \ldots \times P_n$ over $\mathcal{X}^n$. Given $i \leq n$, we denote with $\theta_i(S) = P_i$ the $i$-th component of $S$.

Let $\mathcal{S}_n$ be a family of candidate probability distributions for the (unknown) distribution $S$ of the random vector $\bm{X}_n$. Given an observed $\bm{X}_n \sim S$, our goal is to estimate $P_n = \theta_n(S)$. Let $\hat{\theta}_n = \hat{\theta}_n(\bm{X}_n)$  be an estimator of this property. Given a suitable metric $d(\cdot,\cdot)$ that quantifies the error of the estimation, the \emph{minimax risk} at time $n$ is 
\begin{align}
\label{eq:minimax-risk}
\inf_{\hat{\theta}_n} \sup_{S \in \mathcal{S}_n}  \Exp_{\bm{X}_n \sim S} \left[ d\left( \hat{\theta}_n(\bm{X}_n), \theta_n(S) \right)\right] \enspace , 
\end{align}
where we take the supremum (worst-case) over all the product distributions $S$ in $\mathcal{S}_n$, and the infimum is over all possible estimators $\hat{\theta}_n$. 
The minimax risk quantifies the largest estimation error that the best estimator can possibly achieve with respect to $\mathcal{S}_n$ at a given time $n$. We omit the subscript $n$ when it is clear from the context. For each estimation problem, we will adopt the most used distance in the literature for the specific estimation problem, that is the \emph{total variation distance} for discrete densities (probability mass function), and the $L_2$ \emph{distance} for smooth densities.

We are also interested in the \emph{minimax rate of the average risk}, which quantifies the average of the estimation errors of a online algorithm that at each step observes a new random variable, and outputs an estimate of its distribution based on all the previous observations. Given $i \leq n$, we let $\hat{\theta}_i(\bm{X}_i)$ be an estimator of $\theta_i(S) = P_i$. Given a metric $d$, the minimax rate of the average risk is defined  as
\begin{align}
\label{eq:minimax-cumulative-risk}
\inf_{\hat{\theta}_1,\ldots,\hat{\theta}_n} \sup_{S \in \mathcal{S}_n} \Exp_{\bm{X}_n \sim S} \left[ \frac{1}{n}\sum_{i=1}^n d\left( \hat{\theta}_i(\bm{X}_i), \theta_i(S) \right)\right] \enspace .
\end{align}

Due to space constraints, some of the proofs are deferred to the appendix.

\section{Distribution Drift}
\label{sec:distribution-drift}
The family of candidate probability distributions $\mathcal{S}_n$ is defined by the assumptions on the distribution drift in the stochastic process. The most widely used assumption in the literature is a bound on the drift at each step~\cite{bartlett1992learning}. In our setting, this is formally expressed as follows: there exists $\Delta > 0$ such that $d(P_i,P_{i+1}) \leq \Delta$ for any $i \leq n-1$. In the context of learning binary functions, minimax risk lower bound are known under this setting \cite{barve1996complexity,long1999complexity}. However, this is a very pessimistic assumption, as in the worst case the distance from $P_{i}$ to $P_{n}$ is $(n-i)\Delta$, since we can accumulate an additive error $\Delta$ at each step. 
As an illustrative example that shows the drawback of this assumption, let $\Delta \in (0,1)$ and consider a sequence of distributions $P_1,\ldots,P_n$ over $\mathbb{N}$ such that $P_i$ takes the value $0$ with probability $1-\Delta$, and the value $i$ with probability $\Delta$. This sequence has bounded drift at each step, i.e. the total variation distance between $P_i$ and $P_{i+1}$ is equal to $\Delta$. However, the total variation distance between $P_n$ and $P_i$ is also equal to $\Delta$  rather than $(n-i)\Delta$. The minimax risk subject to only the weak condition $d(P_i,P_{i+1}) \leq \Delta$ can be arbitrarily far from the best estimation error.

An alternative assumption is a polynomial drift \cite{hanneke2015learning}. That is, we assume that there exists a value $\alpha \in [0,1]$ such that $d(P_i,P_n) \leq (n-i)^{\alpha} \Delta$. While this assumption allows to obtain closed-formula upper bounds on the error for the problem of agnostic learning, no lower bounds are known in this setting.

In this work, we introduce and present upper and lower bounds with a more detailed approach for defining drift. 
\begin{definition}
\label{def:regular-sequence}
A vector $\bm{\Delta}_n = (\Delta_1,\ldots,\Delta_n) \in \mathbb{R}^n_{\geq 0}$ is a \emph{regular drift sequence} for a product distribution $S = P_1 \times \ldots \times P_n$ with respect to a metric $d(\cdot,\cdot)$ if:
\begin{compactenum}
\item $\Delta_i$ is an upper bound on the drift between $P_i$ and $P_n$, i.e. $ d(P_i,P_n)\leq \Delta_i$,
\item  the sequence $\Delta_1,\ldots,\Delta_n$ is non-increasing,
\item there is no abrupt change in the drift: there is a constant $c$ such that $\Delta_{i-1}/\Delta_i \leq c$ for any $i =2, \ldots, n-1$,
\item $\Delta_i = 0 \iff i = n$.
\end{compactenum}
\end{definition}

\paragraph{Comment.}  Our results also hold if we substitute in the above definition the requirement that $ d(P_i,P_n)\leq \Delta_i$ with the requirement that $d(P_i,P_{i+1}) \leq \Delta_i - \Delta_{i+1}$ for any $i \leq n-1$. The latter property is stronger as it implies the former, however our lower bound proof works with either definition, and we will us them interchangeably. 

In our work, we characterize the minimax risk of density estimation problems under an arbitrary regular drift sequence $\bm{\Delta}_n$. Noticeably, this is the first work to provide lower bound for estimation problems under such a general model of drift, as the previous lower bound construction assumed a bounded drift at each step \cite{barve1996complexity}.



\section{Discrete Density Estimation}
\label{sec:discrete}
In this section, we show the minimax risk and the minimax rate of the average risk for the problem of estimating discrete distributions with finite support under distribution drift. Without loss of generality, we can let 
the sample space be $\mathcal{X} = [k]$ where $k$ denotes the support size. Since the sample space is discrete, a distribution over $\mathcal{X}$ is a probability mass function $P$, such that $P(j) = \Pr_{X \sim P}(X = j)$ for $j \in [k]$.

Following previous work on estimating discrete distributions, we evaluate the quality of the estimation by the \textit{total variation distance} metric. Given two distribution $P$ and $Q$ over $[k]$, their total variation distance is defined as
\begin{align*}
\TV(P,Q) \doteq \frac{1}{2} \sum_{j \in [k]}| P(j) - Q(j)| \enspace .
\end{align*}
We consider the following family of probability distributions $\mathcal{S}(\bm{\Delta}_n, k)$ over $\mathcal{X}^n$ with regular drift $\bm{\Delta}_n$.

\begin{definition}
\label{def:family-discrete}
Let $\mathcal{X} = [k]$. Let $\bm{\Delta}_n \in \mathbb{R}^n_{\geq 0}$, and let $k > 0$. A product distribution $S = P_1 \times \ldots \times P_n$ over $\mathcal{X}^n$ belongs to $\mathcal{S}_n(\bm{\Delta}_n, k)$ if and only if $\bm{\Delta}_n$ is a regular drift sequence for $S$ with respect to the metric $\TV$.
\end{definition}

We establish the following minimax risk in this setting:

\begin{theorem}
\label{thm:discrete}
Let $\mathcal{S}_n(\bm{\Delta}_n,k)$ be defined as in \Cref{def:family-discrete}, and let 
\begin{align*}
r^* = \max \left\{ r \in [n] : \Delta_{n-r+1} \leq \sqrt{\frac{k}{r}}\right\}
\end{align*}
If $r^*$ is well-defined and $r^* > k$, then:
\begin{align*}
\inf_{\hat{\theta}_n} \sup_{S \in \mathcal{S}_n(\bm{\Delta}_n, k)}  \Exp_{\bm{X}_n \sim S}  \TV\left( \hat{\theta}_n(\bm{X}_n), \theta_n(S) \right) 
= \Theta\left( \sqrt{\frac{k}{r^*}}\right)
\end{align*}
\end{theorem}
This is the first work to characterize the minimax risk for the discrete density estimation problem under distribution drift. In \Cref{sub:discrete-ub}, we analyze a simple algorithm that achieves the upper bound of the theorem. 
Our results extend the known minimax risk of $\Theta(  \sqrt{k/n})$ for estimating a discrete distribution with $n$ independent and identically distributed samples. In fact, for $\bm{\Delta}_n \rightarrow 0$, we have that $r^* = n$.
Noticeably, \Cref{thm:discrete} provides matching upper and lower bound for any regular drift sequence $\bm{\Delta}_n$, and this is the first theoretical work within the distribution drift literature that provides a lower bound in such a general drift setting.  As a simple corollary of this theorem, we can obtain the minimax risk for more specific drift assumptions previously used in the literature.

\textit{Bounded drift at each step.} Assume that there exists a constant $\Delta > 0$ such that for any $S = P_1 \times \ldots \times P_n$, it holds that $\TV(P_i,P_{i+1}) \leq \Delta$. We can invoke \Cref{thm:discrete} with the regular drift sequence $\bm{\Delta}_n = ( \Delta\cdot (n-1), \ldots, \Delta, 0)$. As
$r^* = \max\left\{ r \in [n] : (r-1)\Delta \leq \sqrt{k/r} \right\}$, we can observe that for $n \gtrsim (k/\Delta^2)^{1/3}$, we have that $r^* = \Theta\left((k/\Delta^2)^{1/3}\right)$, and thus the minimax risk in this setting is $\Theta(  (k\Delta)^{1/3})$.

\textit{Polynomial drift.} Assume that there exists a $\alpha \in (0,1]$ such that $\TV(P_i,P_n) \leq (n-i)^{\alpha} \Delta$  for all $i \in [n]$. We can invoke \Cref{thm:discrete} with the regular drift sequence $\bm{\Delta}_n = ( \Delta\cdot (n-1)^{\alpha}, \Delta \cdot (n-2)^{\alpha} \ldots, \Delta, 0)$, and we obtain that for $n \gtrsim (k/\Delta^2)^{1/(2\alpha+1)}$, the minimax risk in this setting is $\Theta\left( (k\Delta)^{1/(2\alpha+1)} \right)$. This is the first work to show a lower bound with this drift assumption.

\subsection{Connection to Agnostic Learning}
\label{sub:agnostic-learning}
We can easily show that the lower bound of \Cref{thm:discrete} also applies to the problem of agnostic learning a family of binary functions with VC dimension $k$. In fact, consider the family of binary functions $\mathcal{F} = \{ f_A : A \subseteq [k] \}$, where $f_A(j) = \mathbf{1}_{\{ j \in A\}}$ for any $j \in [k]$, and observe that the VC-dimension of $\mathcal{F}$ is $k$. Let $\hat{P}$ be any estimator of $P_n$, and let $\hat{P}(A) = \sum_{j \in A}\hat{P}(j)$ for any $A \subseteq [k]$. By using the definition of total variation distance, we have that
\begin{align*}
\sup_{f_A \in \mathcal{F}} \left| \Exp_{X \sim P_n} f_A(X) - \hat{P}(A) \right| = \TV( P_n, \hat{P}) \enspace ,
\end{align*}
which shows that  the problem of estimating the distribution $P_n$ under total variation distance can be reduced to the problem of agnostic learning the family $\mathcal{F}$ with respect to the distribution $P_n$.
We can conclude that the lower bound of \Cref{thm:discrete} applies to the problem of agnostic learning a family of binary functions with VC dimension $k$ in a distribution drift setting. For the case of bounded drift at each step, as observed in the previous subsection, the lower bound is $\Omega( (k \Delta)^{1/3})$ for sufficiently large $n$, and we retrieve the result of \citet{barve1996complexity}. \Cref{thm:discrete} generalizes this lower bound to a more general model of drift, giving tighter bounds when possible, as shown in the example in the Introduction. 

\subsection{Upper Bound}
\label{sub:discrete-ub}
To prove the upper bound of \Cref{thm:discrete} fix a parameter $r\leq n$ and consider the empirical distribution $\hat{P}^{r}$ over the latest $r \leq n$ random variables:
\begin{align}
\label{discrete-estimator}
    \hat{P}^{r}(j) = \frac{1}{r}\sum_{i=n-r+1}^n \mathbf{1}_{\{X_i = j\}} \hspace{20pt} \forall j \in [k] \enspace .
\end{align}
Analogously, we define the average of the latest $r$ distributions as $P^r = (1/r)\sum_{i=n-r+1}^n P_i$. 
In order to evaluate the expected error obtained by using $\hat{P}^{r}$ as an estimate, we use the triangle inequality to decompose the error into two terms
\begin{align}
\label{eqref:discrete-decomposition}
    \Exp \TV(P_n,  \hat{P}^r) \leq \Exp \TV(P^r, \hat{P}^r ) + \TV(P^r,P_n) \enspace .
\end{align}
The first error term of this upper bound is the \emph{statistical error} of
estimating the distribution $P^r$ by its empirical distribution $\hat{P}^r$. 
This error is related to the variance of the estimator which depends on the support size of the estimated distributions. 
\begin{proposition}
\label{prop:discrete-est-error}
$\Exp \TV(P^r, \hat{P}^r )  \leq (1/2)\sqrt{k/r}$. 
\end{proposition}
\vspace{-10pt}
\begin{proof}

By definition $\Exp \TV(P^r, \hat{P}^r ) = (1/2)\sum_{j \in [k]} \Exp| \hat{P}_r(j) - P_r(j)|$, and using Jensen's inequality we have that  for any $j \in \mathbb{N}$, $\Exp| \hat{P}_r(j) - P_r(j)|  \leq \sqrt{ \Var( \hat{P}_r(j) )}$. Since $\hat{P}_r(j)$ is the distribution of an average of $0$-$1$ random variables, we have

\begin{align*}
\sqrt{\Var( \hat{P}_r(j) )} &= \frac{1}{r} \sqrt{\sum_{i=n-r+1}^n P_i(j)(1-P_i(j))} \\ &\leq \frac{1}{r} \sqrt{\sum_{i=n-r+1}^n P_i(j)} = \sqrt{\frac{P^r(j)}{r} }
\end{align*}
Thus, we have
\begin{align*}
\Exp \TV(P^r, \hat{P}^r )\leq \frac{1}{2}\sum_{j \in [k]} \sqrt{\Var( \hat{P}_r(j) )} \leq \frac{1}{2\sqrt{r}}\sum_{j \in [k]}{\sqrt{P^r(j)}}
\end{align*}
We conclude the proof using Cauchy-Schwarz inequality: 
$$\sum_{j \in [k]}\sqrt{P^r(j)} \leq \sqrt{\sum_{j \in [k]}P^r(j)} \sqrt{\sum_{j \in [k]}1} = \sqrt{k} \enspace .$$
\end{proof}

The second error term of the upper bound \eqref{eqref:discrete-decomposition} is the \emph{drift error}, and describes how far the distribution $P^r$ is to $P_n$. Observe that if the samples were identically distributed, this error would be zero. The drift error can be upper bounded by using the information on the drift sequence $\bm{\Delta}_n$.
\begin{proposition}
\label{prop:discrete-drift-error}
$\TV(P^r,P_n) \leq \Delta_{n-r+1}$.
\end{proposition}
\vspace{-10pt}
\begin{proof}
We can rewrite the total variation distance as
\begin{align*}
\TV( P^r,P_n) &= \frac{1}{2}\sum_{j=1}^{k} \sum_{i=n-r+1}^n | P^r(j) - P_n(j)|\\ &\leq \frac{1}{2} \sum_{j=1}^k  \frac{1}{r} \sum_{i=n-r+1}^n | P_i(j) - P_n(j)|
\end{align*}
where in the last inequality we used the triangle inequality and the definition of $P^r$. Therefore, we obtain that $\TV( P^r,P_n) \leq (1/r) \sum_{i=n-r+1}^n \TV(P_i,P_n)$. We conclude the proof by observing that by the monotonicity of the drift sequence $\TV(P_i, P_n) \leq \Delta_i \leq \Delta_{n-r+1}$ for any $i \geq n-r+1$.
\end{proof}
By using \Cref{prop:discrete-est-error} and~\ref{prop:discrete-drift-error}, we can upper bound the estimation error \eqref{eqref:discrete-decomposition} as  $\Exp \TV(P_n,  \hat{P}^r) \leq  (1/2)\sqrt{k/r} + \Delta_{n-r+1}$. There is a trade-off: by choosing a larger $r$, we obtain a smaller statistical error, but potentially a larger drifting error. The value $r^*$ of \Cref{thm:discrete} represents an optimal criterion (up to constants) to solve this trade-off for any regular drift sequence $\bm{\Delta}_n$. The upper bound of the theorem follows by setting the parameter $r$ equal to $r^*$, for which $\Delta_{n-r^*+1} \leq \sqrt{k/r^*}$.


\subsection{Lower Bound}
\label{sub:lb-discrete}
The lower bound of \Cref{thm:discrete} is proven by using information-theoretical tools from minimax theory \cite{yu1997assouad}. In particular, we select a particular family of product probability distributions over $\mathcal{X}^n$ from $\mathcal{S}(\bm{\Delta}_n,k)$, and obtain our lower bound by arguing that the observed values do not provide enough information to distinguish among those distributions (\Cref{assouad}). This family of product probability distributions is constructed as follows. Let $r$ be a parameter such that $1\leq r \leq n$. Each product probability distribution in our family has the same distribution for the first $n-r$ random variables. That is, the first $n-r$ random variables provide no information to decide among the family. For each product distribution in this family, the last $r$ random variables steadily drift in distribution in a distinct direction subject to the constraint of the drift sequence $\bm{\Delta}_n$. We obtain a trade-off: if the value of $r$ is large, it is  easier to decide among the family as we have more time to drift apart, but we make a bigger error if we cannot decide correctly. Conversely, if the value of $r$ is small, it is harder to decide among the family, but we make a smaller error as there is less time to drift apart. Similarly to the upper bound, we obtain a tight lower bound by setting the parameter $r$ equal to $r^*$ as defined in \Cref{thm:discrete}. This choice is adopted throughout this subsection. We distinguish two cases: $(a)$ $r^*=n$ and $(b)$ $r^* < n$. In  case $(a)$, we argue that the minimax error is at least equal to the lower bound $\Omega(\sqrt{k/r^*})$ for discrete density estimation with $n=r^*$ independent and identically distributed samples. In the remaining of this subsection, we focus on case $(b)$.

In order to establish the lower bound for the minimax risk and the minimax cumulative risk, we use Assouad's Lemma as the main technical tool. This is the first work to use this information-theoretic tool to provide lower bounds in a drift setting.
Assouad's Lemma uses a family of probability distribution $\{ S_w : w \in \{0,1\}^m\}$ indexed over a hypercube $\{0,1\}^m$ for some $m \geq 1$. For two binary strings $v,w \in \{0,1\}^m$, their Hamming distance is defined as $h(v,w) = \sum_{i=1}^m \mathbf{1}_{\{v_i \neq w_i\}}$.

\begin{lemma}[Assouad’s Lemma]
\label{assouad}
Let $\theta(\cdot)$ be a target property to estimate. Let $\{ S_w : w \in \{0,1\}^m\} \subseteq \mathcal{S}$ be a family of probability distributions indexed by $w$. Let $p \geq 1$. Then:
\begin{align*}
\inf_{\hat{\theta}} \sup_{S \in \mathcal{S}} \Exp_{\bm{X} \sim S} \left[ 2^p d^p\left( \hat{\theta}(
\bm{X}), \theta(S) \right)\right] \\
\geq  \frac{m}{4} \left( \min_{v \neq w} \frac{d^p(\theta(S_w),\theta(S_v))}{h(v,w)} \right) \bigg[ \min_{\substack{v,w : \\ \lVert w \rVert_1 > \lVert v
\rVert_1 \\ h(v,w) = 1}} e^{-\KL(S_w \| S_v)}\bigg]
 \end{align*}
where $\KL$ is the Kullback–Leibler divergence and $\hat{\theta}$ is any estimator of $\theta(S)$. 
\end{lemma}
Our statement of Assouad's Lemma follows immediately by adapting to our notation its classic statement as in \citep[Lemma 24.3]{van2000asymptotic}. Differently from the latter statement, we state it with the $\KL$-divergence by using the known inequality  $\lVert P \land Q \rVert = \lVert Q \land P\rVert \geq (1/2)\exp(-\KL(P || Q))$ that holds for any distributions $P$ and $Q$. Our formulation is more convenient for the computations of this paper.

Without loss of generality, assume that $k$ is even. Our goal is to properly construct a family of sequence of drifting distributions and apply Assouad's Lemma. We construct a family of product distributions $\{ S_w : w \in \{0,1\}^{k/2} \}$ as follows. For each $w \in \{0,1\}^{k/2}$, we have that $S_w = P_{w,1} \times \ldots \times P_{w,n}$ is the product distributions of $n$ discrete probability distributions over $[k]$. For any $j \in [k]$ and $w \in \{0,1\}^{k/2}$, we define
\begin{align*}
P_{w,i}(j) = \begin{cases}
\frac{1}{k} \hspace{20pt} &\mbox{ if } i \leq n-r^* \\
\frac{1}{k} + (-1)^j w_{\lceil j/2 \rceil}\frac{\Delta_{n-r^*+1}-\Delta_i}{k}  &\mbox{ if } i > n-r^*
\end{cases}
\end{align*}
Intuitively, for any $i 
\geq n-r^*+1$, if $w_j = 1$, then the probabilities of the elements $2j-1$ and $2j$ change as follows: $P_{w,i}(2j-1)$ decreases, while the probability $P_{w,i}(2j)$ increases of the same amount. The following proposition shows that our family of product distributions is well-defined.
\begin{proposition}
\label{prop:discrete-well-defined-set}
    $\{ S_w : w \in \{0,1\}^{k/2} \} \subseteq \mathcal{S}_n(\bm{\Delta}_n, k) $
\end{proposition}
\vspace{-10pt}
\begin{proof}First, we have that each $P_{w,i}$ is a well defined probability distribution for any $w$ and $i$, as 
$\sum_{j \in [k]}P_{w,i}(j)=1$ by construction, and $0 \leq P_{w,i}(j) \leq 1$ for any $j \in [k]$, since $\Delta_{n-r^*+1} \leq \sqrt{k/r^*} < 1$ by assumption of the theorem.
Second, $S_w$ satisfies the assumptions on the drift sequence $\bm{\Delta}_n$ of \Cref{def:family-discrete}. In fact, for any $i < n-r^*+1$, we have that $\TV(P_{w,i},P_{w,i+1}) = 0 \leq \Delta_{i} - \Delta_{i+1}$, and for any $i \geq n-r^*+1$, we have that
\begin{align*}
\TV(P_{w,i}, P_{w,i+1}) &= \frac{1}{2} \sum_{\ell \in [k/2] : w_\ell = 1} \frac{2}{k}\left| \Delta_{i} - \Delta_{i+1}\right| \\
&= \left(\Delta_{i} - \Delta_{i+1}\right)\frac{ \lVert w \rVert_1}{k} \leq \Delta_i - \Delta_{i+1}
\end{align*}
By using the triangle inequality, this also implies that $\TV(P_{w,i},P_{w,n}) \leq \Delta_i$ for any $i \in [n]$.
\end{proof}
Let $\theta_n(\cdot)$ be defined as in \Cref{sec:distribution-drift}. The next two technical propositions show how to compute the quantities required to apply Assouad's Lemma in our setting.

\begin{proposition}
\label{prop:discrete-techn1}
Given $w,w' \in \{0,1\}^{k/2}$, we have that
$\TV(\theta_n(S_w),\theta_n(S_{w'})) = (\Delta_{n-r^*+1}/k)\cdot h(w,w')$.
\end{proposition}
\vspace{-10pt}
\begin{proof}
By definition of $\theta_n(\cdot)$, we have that $\TV(\theta_n(S_w),\theta_n(S_{w'})) = \TV(P_{w,n}, P_{w',n})$. Thus,
\begin{align*}
\TV(P_{w,n}, P_{w',n}) = \frac{1}{2}\frac{\Delta_{n-r^*+1}-\Delta_n}{k} \sum_{\ell=1}^{k/2}2|w'_\ell - w_\ell|,
\end{align*}
and the statement follows by observing that $\Delta_n = 0$ and that $\sum_{\ell}|w'_\ell - w_\ell| = h(w,w')$.
\end{proof}

\begin{proposition}
\label{prop:discrete-techn2}
Let $w,w' \in \{0,1\}^k$ such that $h(w,w') = 1$, and let $w_q \neq w'_q$ be the bit in which they differ. Assume that $w_q=1$. Then $\KL( S_{w} \| S_{w'}) \leq 2$.
\end{proposition}
\vspace{-10pt}
\begin{proof}
By using the factorization property of the 
KL-divergence (see \Cref{prop:factorizationKL} in the appendix), we have that
$\KL( S_{w} \| S_{w'}) =  \sum_{i=n-r^*+1}^n \KL(P_{w,i} \| P_{w',i})$, since $P_{w,i} = P_{w',i}$ for $i < n-r^*+1$.  By using the definition of $\KL$, and the definition of $S_{w}$, we obtain
\begin{align*}
\KL( S_{w} \| S_{w'}) &= \sum_{i=n-r^*+1}^n \Bigg[ P_{w,i}(2q) \log \left(\frac{P_{w,i}(2q)}{P_{w',i}(2q)}\right)  \\ &+ P_{w,i}(2q-1)\log \left(\frac{P_{w,i}(2q-1)}{P_{w',i}(2q-1)}\right) \Bigg] \enspace ,
\end{align*}
as $P_{w,i}$ and $P_{w',i}$ only differ on the elements $2q-1$ and $2q$ for $i \geq n-r^*+1$. If we expand the computation above, we have
\begin{align*}
\KL( S_{w} \| S_{w'}) =
& \frac{1}{k} \sum_{i=n-r+1}^r\Bigg\{ \left( 1 +  \Delta_i\right) \log\left( 1 + \Delta_i \right) \\
&+ \left( 1 - \Delta_i \right) \log\left(1 - \Delta_i \right) \Bigg\} 
\end{align*}
For any $i \geq n-r^*+1$, the following chain of inequality holds $\Delta_i \leq \Delta_{n-r^*+1} \leq \sqrt{k/r^*} < 1$. Thus, we can use the inequality $(1+x)\log(1+x) + (1-x)\log(1-x) \leq 2x^2$ that holds for any $|x|<1$ (see \Cref{prop:auxiliary}). We obtain:
\begin{align*}
\KL( S_{w} \| S_{w'})  \leq  \frac{2}{k} \sum_{i=n-r^*+1}^n \Delta_i^2 \leq (2r^*/k)\Delta^2_{n-r^*+1} 
\end{align*}
where the last inequality follows from the assumption that the sequence $\Delta_1,\ldots,\Delta_n$ is non-increasing. We can conclude the proof by observing that due to the definition of $r^*$, it holds that $\Delta^2_{n-r^*+1} \leq k/r^*$.
\end{proof}

We apply Assouad's Lemma with the family of product distributions $\{ S_w : w \in \{0,1\}^{k/2}\} \subseteq \mathcal{S}_n(\bm{\Delta}_n,k)$, and obtain the following lower bound to the minimax risk
\begin{align*}
\frac{k}{16} \left( \min_{v \neq w} \frac{d(\theta(S_w),\theta(S_v))}{h(v,w)} \right) \bigg( \min_{\substack{v,w : \\ \lVert w \rVert_1 > \lVert v
\rVert_1 \\ h(v,w) = 1}} e^{-\KL(S_w \| S_v)}\bigg)
\end{align*}
We use \Cref{prop:discrete-techn1,prop:discrete-techn2} to lower bound the above expression, and obtain
\begin{align*}
\inf_{\hat{\theta}_n} \sup_{S \in \mathcal{S}_n(\bm{\Delta}_n,k)}\Exp_{\bm{X}_n \sim S}\TV\left( 
\hat{\theta}_n, \theta_n(S)\right) \geq \frac{\Delta_{n-r^*+1}}{16e^2}
\end{align*}
Note that since $r^* < n$, by using the definition of $r^*$, we have that $\Delta_{n-r^*} > \sqrt{k/(r^*+1)}$. Due to the regularity assumption of the drift sequence $\bm{\Delta}_n$, there exists a constant $c$ such that $\Delta_{n-r^*}/\Delta_{n-r^*+1} \leq c$. Therefore, we have that $\Delta_{n-r^*+1} \geq \frac{1}{c} \Delta_{n-r^*} = \Omega(\sqrt{k/r^*})$, and this concludes the proof of the lower bound of \Cref{thm:discrete}.

\subsection{Minimax Rate for the Average Risk}
\label{sub:minimax-discrete-avg}
\Cref{thm:discrete} provides the minimax risk for estimating a distribution at a given time $n$. In this subsection, we want to characterize the minimax rate of the average risk defined as in \eqref{eq:minimax-cumulative-risk} for the online version of this problem. In particular, we want to show that the lower bound proven in \Cref{thm:discrete} for a specific time $n$ is not a rare event but can hold on average for arbitrarily long sequence of estimates.
We study this problem for the case of bounded drift at each steps, i.e. there exists a constant $\Delta > 0$ such that $\TV(P_i,P_{i+1}) \leq \Delta$ for any $i \leq n-1$. Let $\mathcal{S}_n(\Delta,k)$ denote the family of product distributions $S = P_1 \times \ldots \times P_n$ over $\mathcal{X}^{n}$ for which this property holds. The minimax rate of the average risk over $n$ steps is
\begin{align*}
\Pi_n(\Delta,k)\doteq \inf_{\hat{\theta}_1,\ldots,\hat{\theta}_n} \sup_{S \in \mathcal{S}_n(\Delta,k)} \Exp_{\bm{X}_n \sim S}\sum_{i=1}^n \frac{\TV( \hat{\theta}_i(\bm{X}_i), \theta_t(S))}{n}
\end{align*}
and we let $\Pi(\Delta,k) \doteq \lim_{n \rightarrow \infty} \Pi_n(\Delta,k)$. The value $\Pi(\Delta,k)$ represents the limit minimax rate over a arbitrarily large number of steps for the online estimation of a discrete density over $[k]$ with the assumption that the distance of two consecutive distribution is upper bounded by $\Delta$. We can show the following result.
\begin{theorem}
\label{thm:discrete-cumulative-risk}
Let $\Delta \in (0,1/k)$. Then, we have that $
\Pi(\Delta,k) = \Theta( (k \Delta)^{1/3} )$
\end{theorem}
As noted in \Cref{sub:agnostic-learning}, this result also applies for the online problem  of agnostic learning a family of binary classifiers with VC dimension $k$ in the setting of bounded drift at each step.
The upper bound is an obvious corollary of \Cref{thm:discrete}. The difficulty is in obtaining the lower bound.
We note that the lower bound construction of \Cref{sub:lb-discrete} cannot be used to prove a lower bound of the minimax rate for the average risk. In fact, that construction relies on the fact that at a given time $n$, a drift has occurred only for the latest $r^*$ distribution, and it does not provide a tight lower bound for the estimation at all times $i \leq n$.

In order to prove the lower bound of \Cref{thm:discrete-cumulative-risk}, we adopt a different construction. We provide a sketch of the proof. Let $n=m \nu$. We consider product distributions $S = P_1 \times \ldots \times P_n$ that can be partitioned into $m$ blocks of length $\nu$. Let $B_i = P_{(\ell-1)r} \times \ldots \times P_{\ell r}$ be the product distribution of the block $\ell$, i.e.  the distribution of the random variables $(X_{(\ell-1)r}, \ldots, X_{\ell r})$. We let $S$ exhibit a periodic structure. In particular, we guarantee that the first distribution and last distribution of each block $B_\ell$ is a uniform distribution over $[k]$, i.e. $P_{(\ell-1)\nu}$ and $P_{\ell \nu}$ are both uniform distributions for every $\ell \in [m]$. This property plays a double role: it allows us to construct $S$ by considering a sequence of blocks; and the estimation of each block is independent, since samples outside of the block $\ell$ do not help to decide for the distribution $B_\ell$ due to the periodic structure of $S$. 

The proof of the lower bound revolves around the fact that estimating each individual block is hard. In particular, we can show a lower bound of $\Omega\left( (k\Delta)^{1/3}\right)$ to the average error of estimating a block $B_\ell$ given $\bm{X}_{\nu(\ell+1)}$. This result is obtained by using Assouad's Lemma  on a properly defined family of blocks $\mathcal{B}$.
For any sequence of estimators $\hat{\theta}_1,\ldots,\hat{\theta}_n$, we use this previous result to show how to iteratively construct a sequence of blocks $B_1,\ldots,B_m$ from $\mathcal{B}$ such that the the average risk of those estimators with respect to the distribution $S = B_1 \times \ldots \times B_m$ is $\Omega\left( (k\Delta)^{1/3}\right)$. By taking $m \rightarrow \infty$, this is sufficient to prove the lower bound. The details of the full proof are deferred to the appendix.

\section{Smooth Density Estimation}
\label{sec:smooth}
In this section, we establish the minimax risk for the problem of estimating smooth densities under distribution drift. Let our sample space be any arbitrary interval $I \subseteq \mathbb{R}$, i.e. $\mathcal{X} = I$. Given $\bm{X}_n$, our goal is to estimate the density of the distribution $P_n$. In this setting, we also use $P(x)$ to refer to the continuous density of a distribution $P$ at $x \in \mathcal{X}$. Following previous work on nonparametric density estimation, we characterize the smoothness of a density in a Sobolev sense \cite{tsybakov2009introduction}. 
\begin{definition}
\label{def:smooth}
Let $\beta  \in \mathbb{N}_{+}$. A probability density $P$ over $\mathcal{X}$ is $\beta$-smooth if $P$ is differentiable $\beta$ times, $P^{(\beta-1)}$ is absolutely continuous and
$\int \left(P^{(\beta)}(x)\right)^2 dx < \infty$.
\end{definition}

In order to evaluate the error of our estimate, we use the squared $L_2$ distance between densities. Given two densities $f$ and $g$ over $\mathcal{X}$, their $L_2$ distance is defined as
\begin{align*}
L_2(f,g) \doteq \lVert f-g \rVert  = \sqrt{\int_{\mathbb{R}}\left(f(x) - g(x)\right)^2}\enspace .
\end{align*}
In the density estimation literature, the quantity $L_2^2$ is also referred to as \textit{mean integrated squared error}, and it is the most commonly used measure of error. 
In our work, we consider the following family of smooth probability measures $\mathcal{S}(\bm{\Delta}_n,\beta)$ over $\bm{X}_n$ with regular drift $\bm{\Delta}_n$.
\begin{definition}
\label{def:family-smooth}
Let $\bm{\Delta}_n \in \mathbb{R}^n_{\geq 0}$ be a regular drift sequence, and let $\beta > 0$. A product distribution $S = P_1 \times \ldots \times P_n$ over $\mathcal{X}^n$ belongs to $\mathcal{S}(\bm{\Delta}_n, \beta)$ if and only if: $(a)$ $P_i$ is $\beta$-smooth for $i \in [n]$; $\bm{\Delta}_n$ is a regular drift sequence for $S$ with the metric $L_2$.
\end{definition}
We can establish the following minimax risk in this setting.
\begin{theorem}
\label{thm:smooth}
Let $\bm{\Delta}_n \in \mathbb{R}_+^n$ be a regular drift sequence and let $\beta >0 $. Let $\mathcal{S}_n(\bm{\Delta}_n,\beta)$ be defined as in \Cref{def:family-smooth}, and let 
\begin{align*}
r^* = \max \left\{ r \in [n] : \Delta_{n-r+1} \leq \left( \frac{1}{r}\right)^{\frac{\beta}{2\beta+1}}\right\}
\end{align*}
Let $r^* \geq 1$ be well-defined. We have:
\begin{align*}
\inf_{\hat{\theta}_n} \sup_{S \in \mathcal{S}_n(\bm{\Delta}_n, \beta)}  \Exp_{\bm{X}_n \sim S} \lVert \theta_n(S) - \hat{\theta}_n(
\bm{X}_n)\rVert^2  
= \Theta\left( (r^*)^{\frac{-2\beta}{2\beta+1}}\right)
\end{align*}
\end{theorem}
If $\bm{\Delta}_n \rightarrow 0$, we have that $r^* = n$, and we retrieve the known minimax rate $\Theta\left( n^{-\frac{2\beta}{2\beta+1}}\right)$ for estimating a $\beta$-smooth density from $n$ independent and identically distributed samples.
We can achieve the upper bound of \Cref{thm:smooth} with a properly constructed kernel density estimator. A kernel $K$ is a function $K: \mathbb{R} \mapsto \mathbb{R}$ such that $\int K(u) du = 1$. Given a kernel $K$ and a smoothing parameter $h$, the Parzen-Rosenblatt kernel density estimator \cite{rosenblatt1956remarks,parzen1962estimation} over the previous $r$ samples is defined as
\begin{align*}
    \hat{P}^r_{K,h}(x) = \frac{1}{rh}\sum_{i=n-r+1}^n K\left( \frac{X_i-x}{h}\right) \enspace \enspace .
\end{align*}
The parameter $h$ is also referred to as bandwidth. In order to obtain an accurate estimator for highly smooth function, we need to define a special class of kernel functions.

\begin{definition}
\label{def:kernel-l}
Let $\beta \geq 1$ be an integer. We say that $K : \mathbb{R} \mapsto \mathbb{R}$ is a kernel of order $\beta$ if the functions $u \mapsto  u^j K(u)$, with $j = 0,1,\ldots, \beta$ are integrable and satisfy
\begin{align*}
    &\int K(u) du = 1, \hspace{5pt}  \int u^j K(u) du = 0  \hspace{5pt} \text{ for }j = 1 ,\ldots, \beta \\
    & \int K^2(u) du < \infty, \hspace{10pt}  \int |u|^{\beta} |K(u)| du < \infty   \enspace .
\end{align*}
We refer to the work of \citet{tsybakov2009introduction} for a discussion of kernel of order $\beta > 1$. It can be proven that a kernel of order $\beta \geq 2$ cannot be non-negative, and therefore we could obtain an estimate of the density that is negative. This problem can be addressed by taking only the positive part of the estimate, as described in the previous reference.
\end{definition}

If we let $K$ be a kernel of order $\beta$, we can prove that for any $P_1 \times \ldots \times P_n = S \in \mathcal{S}(\bm{\Delta}_n,\beta)$, it holds that 
\begin{align*}
    \Exp \left\lVert \hat{P}^r_{K,h} - P_n \right\rVert^2 = O\left( \Delta^2_{n-r+1} + \frac{1}{r \cdot h} + h^{2\beta}\right) \enspace .
\end{align*}
Observe that the optimal choice of the bandwidth  $h$ to minimize the above upper bound is independent of $\bm{\Delta}_n$. By choosing the value $h = \Theta\left( r^{-1/(2\beta+1)} \right)$, the previous upper bound becomes
\begin{align*}
\Exp \left\lVert \hat{P}^r_{K,h} - P_n \right\rVert^2 = O\left( \Delta^2_{n-r+1} + \left(\frac{1}{r}\right)^{\frac{2\beta}{2\beta+1}} \right) \enspace .
\end{align*}
This bound represents a trade-off between the drift error and the statistical error of the estimation. If we choose $r$ as $r^*$, we obtain the upper bound of the theorem.

The lower bound of the theorem is proven by using a similar strategy to the one used for discrete densities  \Cref{sub:lb-discrete}: we construct a family of product distributions that satisfy the assumption on the drift and use Assouad's Lemma. We refer the details of the proof to the appendix. We point out that it is also possible to prove an average minimax risk result similar to \Cref{thm:discrete-cumulative-risk} for smooth densities by modifying the proof of the discrete case.

\section{Conclusion and Open Questions}
We obtain tight minimax risk bounds for the discrete and  smooth density estimation problems under a general model of distribution drift. We also present the first 
average minimax risk rate in the drift setting. Our results also apply to the important problem of agnostic learning of a family of binary classifiers, improving the known state-of-the-art bounds in the drift setting. 

In this work, we focus on the univariate case for smooth density estimation. Univariate kernel density estimation methods naturally extend to the multivariate setting with similar assumptions \cite{ibragimov1983estimation}. In the i.i.d. case, the minimax rate becomes $O(n^{-2 \beta/(2 \beta + d)})$ with $n$ samples, where $\beta$ is the smoothness of the density and $d$ is the dimensionality of the space. We believe our framework for analysis with drift can provide a characterization of the minimax rate for the multivariate case and this is an interesting future direction.

Another interesting open problem is to provide a competitive algorithm that is oblivious to the drift sequence \cite{hanneke2019statistical}. We refer to \cite{hanneke2015learning} for preliminary results in this direction for the problem of realizable supervised learning under distribution drift.

\textbf{Acknowledgements.} 
This material is based on research sponsored by Defense Advanced Research Projects Agency (DARPA) and Air Force Research Laboratory (AFRL) under agreement number FA8750-19-2-1006 and by the National Science Foundation (NSF) under award
IIS-1813444. The U.S. Government is authorized to reproduce and distribute reprints for Governmental purposes notwithstanding any copyright notation thereon. The views and conclusions contained herein are those of the authors and should not be interpreted as necessarily representing the official policies or endorsements, either expressed or implied, of Defense Advanced Research Projects Agency (DARPA) and Air Force Research Laboratory (AFRL) or the U.S. Government.
\bibliographystyle{icml2023}
\bibliography{icmlsub}

\newpage
\appendix
\onecolumn
\section{Technical Propositions}

\textbf{Definition of KL-divergence.} Let $P$ and $Q$ be two distributions over $\mathcal{X}$. In the continuous case (probability density function, $\mathcal{X} = \mathbb{R}$), their Kullback–Leibler divergence is defined as
\begin{align*}
\KL(P \| Q) = \int_{\mathbb{R}} P(x) \log \left( \frac{P(x)}{Q(x)}\right)dx \enspace .
\end{align*}
In the discrete case (probability mass function, $\mathcal{X}$ is finite), their KL divergence is defined as
\begin{align*}
\KL( P \| Q) = \sum_{x \in \mathcal{X}}P(x) \log \left( \frac{P(x)}{Q(x)}\right)
\end{align*}
The following proposition on the KL-divergence is folklore.
\begin{proposition}[Factorization Property]
\label{prop:factorizationKL}
Let $P$ and $Q$ two distributions over $\mathbb{R}^{d \times n}$ such that the distributions can be factorized, i.e. $P = P_1 \times \ldots \times P_n$ and $Q = Q_1 \times \ldots \times Q_n$. Then, we have that
\begin{align*}
\KL(P \| Q) = \sum_{i=1}^n \KL(P_i \| Q_i)
\end{align*}
\end{proposition}

The following relation will prove useful.
\begin{proposition}
\label{prop:auxiliary}
For any $-1 < x < 1$, we have that
\begin{align*}
(1+x)\log(1+x) + (1-x)\log(1-x) \leq 2x^2
\end{align*}
\end{proposition}

\section{Average minimax risk of online estimation of discrete densities (\Cref{thm:discrete-cumulative-risk})}

\subsection{Proof of the upper bound}
Since the supremum is a convex function, we have that
\begin{align*}
\Pi_n(\Delta,k)  \leq \inf_{\hat{\theta}_1,\ldots,\hat{\theta}_n} \frac{1}{n} \sum_{t=1}^{n} \sup_{S \in \mathcal{S}_t(\Delta,k) }\Exp_{\bm{X}_t \sim S} \TV(\hat{\theta}_t, \theta_t(S))
\end{align*}
Let $\tilde{\theta}_t(\bm{X}_t)$ be the estimator  described in \Cref{sub:discrete-ub} for the estimation at time $t$ that achieves the upper bound of \Cref{thm:discrete}. 
By using those estimators for $t=1,\ldots,n$, the above expression can be upper bounded as
\begin{align*}
\Pi_n(\Delta,k) \leq \frac{1}{n} \sum_{t=1}^{n} \sup_{S \in \mathcal{S}_t(\Delta,k) }\Exp_{\bm{X}_t \sim S} \TV(\tilde{\theta}_t(\bm{X}_t), \theta_t(S))
\end{align*}

As previously discussed, for any $t \gtrsim (k/\Delta^2)^{1/3}$, we have that the expected error of the estimator $\tilde{\theta}_t(\bm{X}_t)$ to estimate $P_t$ is upper bounded by $O( (k \Delta)^{1/3})$ in the case of a bounded drift $\Delta$ at each step. By taking $n \rightarrow \infty$, we can conclude that $\Pi(\Delta,k) = O( (k \Delta)^{1/3})$.

\subsection{Proof of the lower bound} Let $n = \nu m$, and let $\nu = 2(k/\Delta^2)^{1/3}$. Let  $\mathcal{B} = \{ B_w : w \in \{0,1\}^{k/2} \}$ be a family of product distributions over $\mathcal{X}^{\nu}$ constructed as follows. The set $\mathcal{B}$ represents the family of possible candidate block distributions that we will use to build our lower bound. For any $B_w = P'_{w,1} \times \ldots \times P'_{w,\nu}$ and $j \in [k]$, we let 
\begin{align}
\label{eq:def-block}
P'_{w,i}(j) &=  \begin{cases}\frac{1}{k} + (-1)^j w_{\lceil j/2 \rceil}\frac{\Delta (i-1)}{k}   \hspace{10pt} &\mbox{ if } i \leq \nu/2 \\
P'_{w,\nu-i+1}(j)  &\mbox{ if } i > \nu/2 \end{cases}
\end{align}
We can observe that the first $\nu/2$ components of each product distribution share similarities with the family of product distributions defined for the lower bound construction of \Cref{sub:lb-discrete} for the special case of bounded drift at each step. 

This is a properly defined family of product distributions. In fact, for any $w$ and $i$, we have that  $\sum_{j} P_{w,i}(j) = 1$. Moreover, we have that $|(-1)^j w_{\lceil j/2 \rceil} \Delta (i-1)/k| \leq \Delta \nu/(2k) = \Delta^{1/3}k^{1/3}/k \leq 1/k$ due to the assumption $\Delta \in (0,1/k)$. Hence, $P_{w,i}(j) \in [0,1]$.

We can also observe that $\mathcal{B} \subseteq \mathcal{S}_{\nu}(\Delta,k)$. In fact, for any two consecutive distributions $P_{w,i}$ and $P_{w,i+1}$, we have that
\begin{align*}
\TV(P_{w,i}, P_{w,i+1}) \leq \frac{1}{2} \sum_{j = 1}^{k/2} \frac{2 \Delta}{k} \mathbf{1}_{\{w_j=1\}} \leq \Delta
\end{align*}
Observe that for any sequence of blocks $B_1,\ldots,B_{\ell} \in \mathcal{B}$, the product distribution $B_1 \times \ldots \times B_{\ell} \in \mathcal{S}_{\nu \ell}(\Delta,k)$. This property will be exploited later in the proof.

The next lemma shows that estimating the last block of a distribution $S = B_1 \times \ldots \times B_m$ is hard for any $m \geq 1$. It is proven by using Assouad's Lemma, and it is the core technical result that enables the analysis. 
\begin{lemma}
\label{lemma:block-lower-bound}
Let $\ell \in [m]$. Let $B_1,\ldots, B_{\ell-1}$ be any $\ell-1$ blocks from $\mathcal{B}$, and let $S' = B_1 \times \ldots \times B_{\ell-1}$. We have:
\begin{align*}
\inf_{ \hat{\theta}_{\nu(\ell-1)+1}, \ldots, \hat{\theta}_{\nu \ell}} \max_{\substack{S = S' \times B : \\ B \in \mathcal{B}}} \Exp_{\bm{X}_{\ell \nu} \sim S} \left[ \sum_{i = \nu(\ell-1)+1}^{\nu \ell} \frac{\TV( \hat{\theta}_i(\bm{X}_i), \theta_i(S))}{\nu} 
\right]= \Omega\left( (k\Delta)^{1/3}\right) 
\end{align*}
\end{lemma}
\begin{proof}
It is convenient to rewrite the left-hand side of the equation of the lemma. Let $V = V_1 \times \ldots \times V_\nu$ and $W = W_1 \times \ldots \times W_\nu$ be two product distributions. We define the distance  $d_b(V,W) \doteq (1/\nu) \sum_{i=1}^{\nu}\TV(W_i,V_i)$. Given a product distribution $S = P_1 \times \ldots \times P_{\ell \nu}$, we define $\theta_{\#}(S) = P_{(\ell-1)\nu+ 1} \times \ldots \times P_{\ell \nu}$ as the product distribution of the last $\nu$ components of $S$.
Let $\overline{\theta}(\bm{X}_{\nu\ell})$ be any estimator of the product distribution $\theta_{\#}(S) \in \mathcal{B}$. Observe that $\overline{\theta}(\bm{X}_{\ell \nu}) = \hat{\theta}_{\nu(\ell-1)+1}( \bm{X}_{\nu(\ell-1)+1}) \times \ldots \times  \hat{\theta}_{\nu \ell}(\bm{X}_{\nu\ell})$  is a possible estimator for $\theta_{\#}(S)$. Hence, we have that
\begin{align}
\label{eq:temporaneo1}
\inf_{ \hat{\theta}_{\nu(\ell-1)+1}, \ldots, \hat{\theta}_{\nu \ell}} \max_{\substack{S = S' \times B : \\ B \in \mathcal{B}}} \Exp_{\bm{X}_{\ell \nu} \sim S} \left[ \sum_{i = \nu(\ell-1)+1}^{\nu \ell} \frac{\TV( \hat{\theta}_i(\bm{X}_i), \theta_i(S))}{\nu} 
\right] \geq  \inf_{ \overline{\theta}} \max_{\substack{S = S' \times B : \\ B \in \mathcal{B}}} \Exp_{\bm{X}_{\ell \nu} \sim S}  d_b( \overline{\theta}(\bm{X}_{\nu \ell}), \theta_{\#}(S))
\end{align}
Consider the family of product distributions $\{ S_w = S' \times B_w : w \in \{0,1\}^{k/2}\}$ over the hypercube $\{0,1\}^{k/2}$, where $S'$ is defined in the statement of the Lemma. We want to invoke Assouad's Lemma on the estimation problem
\begin{align*}
\inf_{ \overline{\theta}} \max_{S_w : w \in \{0,1\}^{k/2}} \Exp_{\bm{X}_{\ell \nu} \sim S}  d_b( \overline{\theta}(\bm{X}_{\nu \ell}), \theta_{\#}(S))
\end{align*}
that is equal to the right-hand side of \eqref{eq:temporaneo1}.

Let $w,w' \in \{0,1\}^k$. We want to compute $d_b( \theta_{\#}(S_w), \theta_{\#}(S_{w'}))$. By using the definition of $\theta_{\#}$, we have that $d_b( \theta_{\#}(S_w), \theta_{\#}(S_{w'})) = d_b( B_w, B_{w'})$. Let $B_{w} = P'_{w,1} \times \ldots \times P'_{w,\nu}$ and $B_{w'} = P'_{w',1} \times \ldots \times P'_{w',\nu}$ be defined as in \eqref{eq:def-block}. We have that
\begin{align*}
d_b( \theta_{\#}(S_w), \theta_{\#}(S_{w'})) =  \frac{1}{\nu} \sum_{i=1}^{\nu} \TV( P'_{w,i}, P'_{w',i}) = \frac{2}{\nu} \sum_{i=1}^{\nu/2}\TV( P'_{w,i}, P'_{w',i}) &= \frac{1}{\nu} \sum_{i=1}^{\nu/2} \frac{2 (i-1)\Delta}{k} h(w,w') \\
& \geq \frac{\nu \Delta}{2k} h(w,w')
\end{align*}

Let $w,w' \in \{0,1\}^k$ such that $h(w,w') = 1$. Let $q$ denote the bit in which $w$ and $w'$ differ, and assume that $w_q = 1$ and $w'_q=0$. We have that
\begin{align*}
\KL(S_w \| S_{w'}) =  \sum_{i=1}^{\nu} \KL( P'_{w,i} \| P'_{w',i}) 
\end{align*}
due to the factorization property of the KL-divergence \Cref{prop:factorizationKL}, and the fact that the first $\nu(\ell-1)$ components of the product distributions $S_w$ and $S_{w'}$ coincide by construction. We have that
\begin{align*}
\sum_{i=1}^{\nu} \KL( P'_{w,i} \| P'_{w',i}) = \sum_{i=1}^{\nu} \KL( P'_{w,i} \| P'_{w',i}) & =2 \sum_{i=1}^{\nu/2} \sum_{j =1}^k P'_{w,i}(j) \log\left( \frac{P'_{w,i}(j)}{P'_{w',i}(j)} \right) \\ 
\end{align*}
Since $w$ and $w'$ only differ on the bit $q$, we have $P'_{w,i}(j) \neq P'_{w',i}(j)$ only if $j = 2q$ or $j = 2q-1$. By using the definition of $P'_{w,i}$ and $P'_{w',i}$, we obtain
\begin{align*}
\KL(S_w \| S_{w'}) &= 2 \sum_{i=0}^{\nu/2-1}\left\{ \left( \frac{1}{k} +  \frac{\Delta i}{ k}\right) \log\left( 1 + i \Delta \right) + \left( \frac{1}{k} - \frac{\Delta i}{k}\right) \log\left(1 - {i \Delta}\right) \right\} \\
&= \frac{2}{k} \sum_{i=0}^{\nu/2-1}\left\{ \left( 1 +  {\Delta i}\right) \log\left( 1 + i \Delta \right) + \left( 1 - {\Delta i}\right) \log\left(1 - {i \Delta}\right) \right\} \\
&\leq \frac{4\Delta^2}{k} \sum_{i = 0}^{\nu/2-1}i^2 = O\left( \frac{\Delta^2 \nu^3}{k} \right)
\end{align*}
where in the first inequality we used \Cref{prop:auxiliary}. By definition of $\nu$, we have that $\nu^3 = 8k/\Delta^2$, hence  $\KL(S_w \| S_{w'}) = O(1)$. Therefore, we can apply Assouad's Lemma (\Cref{assouad}), and obtain that
\begin{align*}
\inf_{ \overline{\theta}} \max_{S_w : w \in \{0,1\}^{k/2}} \Exp_{\bm{X}_{\ell \nu} \sim S}  d_b( \overline{\theta}(\bm{X}_{\nu \ell}), \theta_{\#}(S)) \geq \frac{\nu \Delta}{16 e^{O(1)}} = \Omega\left( (k \Delta)^{1/3}\right)
\end{align*}
where the last equality follows by substituting the definition of $\nu$.
\end{proof}

The lower bound of \Cref{thm:discrete-cumulative-risk} is obtained by iteratively applying \Cref{lemma:block-lower-bound} to construct a hard distribution given any sequence of estimators $\hat{\theta}_1, \ldots \hat{\theta}_n$.

\begin{proof}[Proof of the lower bound of \Cref{thm:discrete-cumulative-risk}]
Let $n = \nu m$, and let $\nu$ and $\mathcal{B}$ be defined as within this section. We want to show a lower bound to
\begin{align*}
\Pi_n(k,\Delta) = \inf_{\hat{\theta}_1,\ldots,\hat{\theta}_n} \sup_{S \in \mathcal{S}_n(\Delta,k)} \Exp_{\bm{X}_n \sim S}\sum_{t=1}^n \frac{\TV( \hat{\theta}_t(\bm{X}_t), \theta_t(S))}{n}
\end{align*}
We prove a lower bound as follows. For any sequence of estimators $\hat{\theta}_1,\ldots,\hat{\theta}_n$, we show how to construct a sequence of blocks $S^* = B_1^* \times \ldots \times B^*_m$ from $\mathcal{B}$ such that
\begin{align*}
\Exp_{\bm{X}_n \sim S^*}\sum_{t=1}^n \frac{\TV( \hat{\theta}_t(\bm{X}_t), \theta_t(S^*))}{n} = \Omega\left( (k \Delta)^{1/3}\right) \enspace.
\end{align*}
Since $S^* \in \mathcal{S}_{n}(\Delta,k)$, this is sufficient to prove the lower bound of the theorem.

Fix a sequence of estimators $\hat{\theta}_1,\ldots,\hat{\theta}_n$. We construct $B_1^*, \ldots, B^*_m$ iteratively as follows. We let
\begin{align*}
B^*_1  = \argmax_{\substack{B \in \mathcal{B}}} \Exp_{\bm{X}_{\nu} \sim B} \left[ \sum_{i = 1}^{\nu} \frac{\TV( \hat{\theta}_i(\bm{X}_i), \theta_i(B))}{\nu} 
\right]
\end{align*}
For any $\ell \in \{2 ,\ldots ,m\}$, we let $S^*_{(\ell-1)} = B_1^* \times \ldots \times B_{\ell-1}^*$ and 
\begin{align*}
B^*_{\ell} = \argmax_{B \in \mathcal{B}}\Exp_{\bm{X}_{\ell \nu} \sim S^*_{(\ell-1)}  \times B} \left[ \sum_{i = \nu(\ell-1)+1}^{\nu \ell} \frac{\TV( \hat{\theta}_i(\bm{X}_i), \theta_i(S^*_{(\ell-1)} \times B))}{\nu} 
\right]
\end{align*}
Let $S^* = S^*_{(m)} = B^*_1 \times \ldots \times B^*_m$. We have that due to the way that  $B_1^* \times \ldots \times B_{\ell-1}^*$ are defined, \Cref{lemma:block-lower-bound} applies, and for any $\ell \in [m]$, it holds
\begin{align}
\label{eq:tmptmptmp11}
\Exp_{\bm{X}_{\ell \nu} \sim S^*_{(\ell)}} \left[ \sum_{i = \nu(\ell-1)+1}^{\nu \ell} \frac{\TV( \hat{\theta}_i(\bm{X}_i), \theta_i(S^*_{(\ell)}))}{\nu}  \right] = \Omega\left( (k \Delta)^{1/3}\right)
\end{align}
By using linearity of expectation, we have that
\begin{align*}
\Exp_{\bm{X}_n \sim S^*}\sum_{t=1}^n \frac{\TV( \hat{\theta}_t(\bm{X}_t), \theta_t(S^*))}{n}  = \frac{1}{m}\sum_{\ell=1}^{m} \Exp_{\bm{X}_{\nu \ell} \sim S^*_{(\ell)} } \left[ \frac{1}{\nu} \sum_{i = (\ell-1)\nu+1}^{\nu \ell}\TV( \hat{\theta}_i(\bm{X}_i), \theta_i(S^*_{(\ell)})\right] &= \frac{1}{m}\sum_{\ell=1}^m \Omega\left( (k \Delta)^{1/3}\right) \\&= \Omega\left( (k \Delta)^{1/3}\right)  \enspace .
\end{align*}
\end{proof}

\section{Smooth Density Estimation under Distribution Drift}

\newcommand{\fe}{\hat{f}_{\bm{\alpha}}}
\newcommand{\ft}{f_{\bm{\alpha}}}
\newcommand{\pe}{\hat{P}^r_{K,h}}
\newcommand{\pt}{P_{[r]}}
\newcommand{\SD}{\mathcal{P}_{\mathcal{H}}(\beta,L)}

\subsection{Upper Bound of \Cref{thm:smooth}}
The structure of the proof of the upper bound uses technical ideas from the analysis of kernel density estimator in a non drift setting \cite{tsybakov2009introduction}. Due to the length of the analysis, the proof is broken down into multiple lemmas. For the remaining of this subsection, let $P_1 \times \ldots \times P_n = S \in \mathcal{S}_n(\bm{\Delta}_n,\beta)$, where $\mathcal{S}$ is defined as in \Cref{def:family-smooth}. 
For any $1 \leq r \leq n$, let $\pt$ be the average of the densities $P_{n-r+1}, \ldots, P_n$, i.e. $\pt = (1/r) \sum_{i=n-r+1}^n P_i$.

For any $x \in \mathbb{R}$, we define the following quantities
\begin{align*}
    &b(x)  = \Exp_{\bm{X}_n \sim S}\left[  \pe(x)\right] - \pt(x) \hspace{50pt}& \text{Bias of the estimator}  \\
    &\sigma^2(x) = \Exp_{\bm{X}_n \sim S}  \left(  \pe(x)  - \Exp_{\bm{X}_n \sim S} \pe(x) \right)^2 & \text{Variance of the estimator} \\
    &d^2(x) = \left( \pt(x) - P_n(x) \right)^2 & \text{Drift error}
\end{align*}

We can obtain the following error decomposition based on the above quantities.

\begin{proposition}[Error Decomposition]
\label{ubkde:prop1}
We have that
\begin{align*}
   \Exp \lVert P_n - \pe \rVert^2 = \Exp\int \left(\pe(x) - P_n(x)\right)^2 dx \leq  2\int d^2(x)dx + 2 \int b^2(x) dx + 2\int \sigma^2(x)dx
\end{align*}
\end{proposition}
\begin{proof}
By Tonelli-Fubini theorem, we can swap the integral and the expectation. We observe that
\begin{align*}
    \Exp \left(\pe(x) - P_n(x)\right)^2 &= \Exp \left( \pe(x) - \pt(x) + \pt(x) - P_n(x) \right)^2 \\
    &\leq 2\Exp( \pe(x) - \pt(x))^2 + 2( \pt(x) - P_n(x))^2 \\
    &= 2\Exp( \pe(x) - \pt(x))^2 + 2d^2(x)
\end{align*}
The first inequality is due to the fact that $(x+y)^2 \leq 2x^2+2y^2$ for any $x,y \in \mathbb{R}$. We can use bias-variance decomposition and obtain that
\begin{align*}
    \Exp( \pe(x) - \pt(x))^2 =  b^2(x) + \sigma^2(x) \enspace .
\end{align*}
\end{proof}

In the next three propositions, we will now individually upper bound each source of error.
\begin{proposition}[Upper bound on drift error]
\label{ubkde:prop2}
\begin{align*}
    \int d^2(x) dx = O(\Delta^2_{n-r+1})
\end{align*}
\end{proposition}
\begin{proof}
To prove the statement, we use Cauchy-Schwarz inequality:
\begin{align*}
    [\pt(x) - P_n(x)]^2 &= \left[ \sum_{i =n-r+1}^n \frac{1}{r} \left( P_i(x) - P_n(x) \right) \right]^2  \\
    &= \left[ \sum_{i =n-r+1}^n  \sqrt{\frac{1}{r}}\sqrt{\frac{1}{r}} \left( P_i(x) - P_n(x) \right) \right]^2 \\
    & \leq \left[ \sqrt{\sum_{i =n-r+1}^n 1/r} \cdot \sqrt{\sum_{i =n-r+1}^n (1/r)(P_i(x) - P_n(x))^2} \right]^2 \\
    &= \sum_{i =n-r+1}^n \frac{1}{r} |P_i(x) - P_n(x)|^2
\end{align*}
Therefore, we have that
\begin{align*}
    \int d^2(x) dx &\leq \int\sum_{i =n-r+1}^n \frac{1}{r} |P_i(x) - P_n(x)|^2dx  \\
    &= \sum_{i =n-r+1}^n  \frac{1}{r}\int | P_i(x) - P_n(x)|^2 dx  \\
    &\leq \sum_{i =n-r+1}^n  \frac{\Delta_i^2}{r}
\end{align*}
where the last inequality is due to the assumption on the regular drift $\bm{\Delta}_n$. Since $\Delta_{n-r+1} \leq \Delta_i$ for any $i \geq n-r+1$, we can conclude that $\int d^2(x) dx \leq \Delta^2_{n-r+1}$ \enspace .
\end{proof}

The next two propositions follow by a slight modification of the the proofs from \citet[Proposition 1.4 and 1.5]{tsybakov2009introduction}, as we adapt those results in the setting where each random variable is sampled from a different distribution. For completeness, we report the full proofs, and refer the reader to the previous reference for additional details.
\begin{proposition}[Upper bound on variance error]
\label{ubkde:prop3}
Suppose that the kernel $K$ satisfies 
$m_K = \int K^2(u)du < \infty$.
Then:
\begin{align*}
    \int \sigma^2(x) dx \leq \frac{m_K}{r \cdot h}
\end{align*}
\end{proposition}
\begin{proof}
For $i \in [n]$ and $x \in \mathcal{X}$, consider the random variables $\eta_i(x) = K((X_i - x)/h)-\Exp K((X_i - x)/h)$. The random variables $\eta_1(x), \ldots, \eta_n(x)$ are independent and have zero mean, and their variance can be upper bounded as
\begin{align*}
\eta^2_i(x)  \leq \Exp K^2\left(\frac{X_i - x}{h}\right)  \hspace{50pt} \forall i \in [n] \enspace .
\end{align*}
Therefore, we have that
\begin{align*}
\sigma^2(x) = \Exp \left( \frac{1}{rh}\sum_{i=n-r+1}^n  \eta_i(x) \right)^2 = \frac{1}{r^2 h^2} \sum_{i=n-r+1}^n \Exp \eta_i^2(x)  \leq \frac{1}{r^2 h^2} \sum_{i=n-r+1}^n \Exp K^2\left(\frac{X_i - x}{h}\right) \enspace . 
\end{align*}
By integrating the above expression, we obtain that
\begin{align*}
\int \sigma^2(x) dx  &\leq \frac{1}{r^2 h^2}\sum_{i =n-r+1}^n \left( \int \Exp  K^2\left(\frac{X_i - x}{h}\right)dx\right) \\
&= \frac{1}{r^2 h^2} \sum_{i=n-r+1}^n\int  \left[ \int K^2\left(\frac{z_i-x}{h}\right)P_i(z_i)dz_i \right] dx \\
&= \frac{1}{r^2 h^2} \sum_{i=n-r+1}^n\int  \left[ \int K^2\left(\frac{z_i-x}{h}\right)dx \right] P_i(z_i)dz_i \\
&=  \frac{1}{r^2 \cdot h} \sum_{i=n-r+1}^n\int K^2(u) du\\
&= \frac{m_K}{rh} \enspace .
\end{align*}
\end{proof}


We remind that the Taylor expansion of a function $f$ on $\mathbb{R}$  that is differentiable $\beta$ times in each point of its domain can be written as follows. Let $x \in \mathbb{R}$, $u \in \mathbb{R}$, and $h>0$, then
\begin{align*}
    f(x + uh) = f(x) + f'(x)uh + \ldots + \frac{(uh)^\beta}{(\beta-1)!} \int_{0}^{1}(1 - \tau)^{\beta -1}f^{(\beta)}(x + \tau u h)d\tau \\
\end{align*}
\begin{proposition}[Upper bound Bias Error]
\label{ubkde:prop4}
Let $K$ be a kerner of order $\beta$. We have that
\begin{align*}
    \int b^2(x)dx = O\left(h^{2\beta}\right)
\end{align*}
\end{proposition}
\begin{proof}
Consider the bias error term
\begin{align*}
    b(x) = \Exp\left[ \pe(x) \right] - \pt(x)
\end{align*}
The function $b(x)$ can be rewritten as
\begin{align*}
    b(x) &= \frac{1}{h}\sum_{i =n-r+1}^n\frac{1}{r} \left\{
    \left[\int K \left( \frac{z - x}{h}\right)P_i(z)dz \right] - h P_i(x) \right\}\\
    &= \sum_{i =n-r+1}^n\frac{1}{r}  \int K(u)\left[P_i(x + uh) - P_i(x) \right] du
\end{align*}
By using the Taylor expansion of $P_i$ and the fact that $K$ is a kernel of order $\beta$, we obtain that
\begin{align*}
    b(x) = \sum_{i =n-r+1}^n\frac{1}{r} \int K(u) \frac{(uh)^{\beta}}{(\beta-1)!} \left[\int_{0}^{1}(1-\tau)^{\beta -1} P_i^{(\beta)}(x + \tau u h) d\tau \right]du
\end{align*}
Since the Kernel is of order $\beta$, we have that $\int K(u) u^\beta P^{(\beta)}(x) du = 0$ for any density $P$, therefore we have that
\begin{align*}
     b(x) &=  \sum_{i =n-r+1}^n\frac{1}{r} \int K(u) \frac{(uh)^{\beta}}{(\beta-1)!} \left[\int_{0}^{1}(1-\tau)^{\beta -1} P_i^{(\beta)}(x + \tau u h) d\tau \right]du \\
     &=  \sum_{i =n-r+1}^n\frac{1}{r} \int K(u) \frac{(uh)^{\beta}}{(\beta-1)!} \left[\int_{0}^{1}(1-\tau)^{\beta -1} \left( P_i^{(\beta)}(x + \tau u h) - P_i^{(\beta)}(x) \right) d\tau \right]du
\end{align*}
We use Cauchy-Schwarz inequality to separate the contribution of each $P_i$, and obtain that
\begin{align*}
    b^2(x) &=  \left(\sum_{i =n-r+1}^n\frac{1}{r} \int K(u) \frac{(uh)^{\beta}}{(\beta-1)!} \left[\int_{0}^{1}(1-\tau)^{\beta -1} \left( P_i^{(\beta)}(x + \tau u h) - P_i^{(\beta)}(x) \right) d\tau \right]du\right)^2  \\
    &\leq \sum_{i =n-r+1}^n\frac{1}{r}  \left(\int K(u) \frac{(uh)^{\beta}}{(\beta-1)!} \left[\int_{0}^{1}(1-\tau)^{\beta -1} \left( P_i^{(\beta)}(x + \tau u h) - P_i^{(\beta)}(x) \right) d\tau \right]du\right)^2 
\end{align*}
We use the above inequality to upper bound $\int b^2(x)dx$ as follows
\begin{align*}
    \int b^2(x) dx &\leq  \sum_{i =n-r+1}^n\frac{1}{r} \int\left(\int K(u) \frac{(uh)^{\beta}}{(\beta-1)!} \left[\int_{0}^{1}(1-\tau)^{\beta -1} \left|P_i^{(\beta)}(x + \tau u h) - P_i^{(\beta)}(x) \right| d\tau \right]du\right)^2  dx
\end{align*}
Now, we can proceed as in \citet[Proposition 1.5]{tsybakov2009introduction} to show an upper bound of $O(h^{2\beta})$ for each integral within the sum over $i$.  By doing so, we can conclude that
\begin{align*}
    \int b^2(x) dx =O\left(\sum_{i =n-r+1}^n \frac{1}{r} h^{2\beta} \right)=O\left( h^{2\beta} \right)
\end{align*}
\end{proof}

\subsubsection{Proof of the upper bound of \Cref{thm:smooth}}
Let $S \in \mathcal{S}_n(\bm{\Delta}_n,\beta)$. Let $K$ be a Kernel of order $\beta$ as in \Cref{def:kernel-l}, and let $h>0$. We consider the estimator $\hat{P}^r_{K,h}$ defined as in \Cref{sec:smooth}. We use \Cref{ubkde:prop1} and have that
\begin{align*}
\Exp \lVert P_n - \pe \rVert^2 \leq 2\int \left( d^2(x) + b^2(x) + \sigma^2(x) \right)dx
\end{align*}
We upper bound each term of the right-hand side of the above inequality with Propositions~\ref{ubkde:prop2},~\ref{ubkde:prop3} and~\ref{ubkde:prop4}:
\begin{align*}
    \Exp \lVert P_n - \pe \rVert^2 = O\left( h^{2\beta} + \frac{1}{rh} + \Delta^2_{n-r+1} \right)
\end{align*}
We choose the bandwidth as $h =r^{-1/(2\beta+1)}$. In this case, the above upper bound becomes
\begin{align*}
\Exp \lVert P_n - \pe \rVert^2 = O\left( r^{-2\beta/(\beta+1)}+ \Delta^2_{n-r+1} \right)
\end{align*}
If we use as $r$ the value $r^*$ defined as in the theorem statement, it holds that $\Delta^2_{n-r^*+1} \leq (r^*)^{-2
\beta/(2\beta+1)}$. We can conclude that $\Exp \lVert P_n - \pe \rVert^2 = O\left( (1/r^*)^{2\beta/(2\beta+1)}\right)$

\qed

\subsection{Lower Bound of \Cref{thm:smooth}}
\label{sub:lb-smooth}
Let $K: \mathbb{R} \mapsto [0, \infty)$ be a function such that $K(x) > 0$ only if $x \in (-1/2,1/2)$, $\int K(x)dx = 0$, $K$ is infinitely differentiable,
 $K$ is $\beta$-smooth,
 $\lVert K \rVert_{\infty} = \sup_{x}|K(x)| \leq 1$, and
 $\lVert K \rVert_2 =  \sqrt{\int K^2(x) dx} \leq 1$. By following the example of \citet[p.93]{tsybakov2009introduction}, we adopt the function
 \begin{align}
 \label{def-K-lb}
K(x) = K_0(4x+1) - K_0(4x-1)
\end{align}
where
\begin{align*}
    K_0(x) = \exp\left( -\frac{1}{1-u^2}\right)\mathbf{1}_{\{ -1 < x < 1 \}}   \enspace .
\end{align*}
\begin{figure}[ht]
\begin{center}
\centerline{\includegraphics[width=\columnwidth]{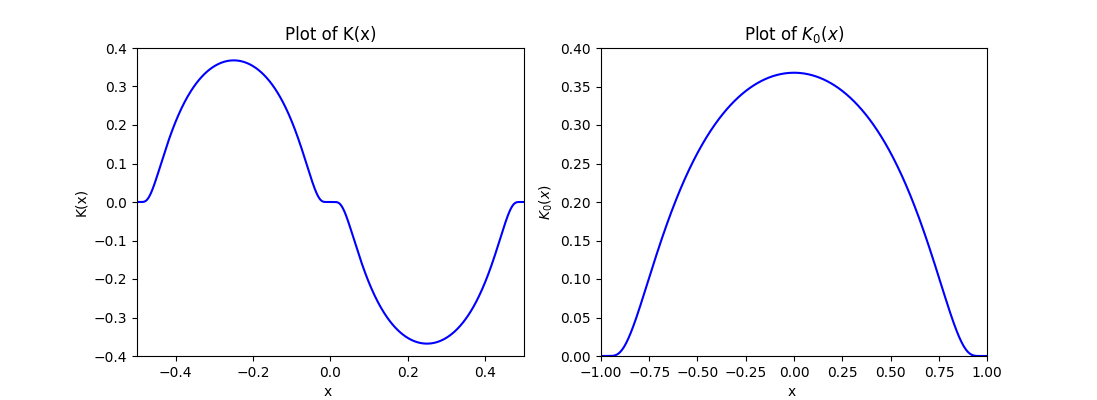}}
\caption{Plots of the functions $K(x)$ and $K_0(x)$}
\label{icml-historical}
\end{center}
\label{fig:plot-K}
\end{figure}

The proof of the lower bound follows the same skeleton of the method described in \Cref{sub:lb-discrete} for discrete distributions.

We consider a sample space $\mathcal{X} = [0,1]$. Let $m =(r^*)^{1/(2\beta+1)} \geq 1$. Let $x_j = (2j-1)/(2m)$ for $j \in [m]$. We construct a family of product distributions $\{S_w = P_{w,1} \ldots \times P_{w,n} : w \in \{0,1\}^m \}$  over $\mathcal{X}^n$ defined as follows
\begin{align*}
    &P_{w,i}(x) = \begin{cases}\mathbf{1}_{\{0 \leq x \leq 1\}}     \hspace{50pt} &  \mbox{ if } i < n-r^*+1 \\
    \mathbf{1}_{\{0 \leq x \leq 1\}} + \sum_{j=1}^m w_j \frac{(
    \Delta_{n-r^*+1} - \Delta_i)}{m^\beta} K(m(x - x_j)) & \mbox{ if } i 
 \geq n-r^*+1 \end{cases}
\end{align*}
This construction exhibits the following crucial property. For every $x \in [0,1]$, there exists at most one value $j \in [m]$ such that $K( m(x-x_j)) > 0$. This property will be used throughout this section.

In the next three propositions, we show that the family $\{ S_w : w \in \{0,1\}^m \}$ is well-defined, and we compute the quantities required to apply Assouad's Lemma.
\begin{proposition}
\label{prop:lb-smooth-belong}
We have that $\{ S_w : w \in \{0,1\}^m \} \subseteq \mathcal{S}_n( \bm{\Delta}_n, \beta)$.
\end{proposition}
\begin{proof}
For any $w \in w \in \{0,1\}^m$ and $i < n-r^*+1$, we have that $P_{w,i}$ is trivially a well-defined density. If $i \geq n-r^*+1$, we have that
\begin{align*}
\int P_{w,i}(x)dx = \int \mathbf{1}_{\{0 \leq x \leq 1\}}dx +  \sum_{j=1}^m w_j \frac{(
    \Delta_{n-r^*+1} - \Delta_i)}{m^\beta} \int K(m(x - x_j)) dx  = 1
\end{align*}
where the last equality is due to the fact that $\int K(x) dx = 0$. We also want to prove that $P_{w,i}(x) \geq 0$ (the density is non-negative). We can show that for any $x \in [0,1]$, we have that
\begin{align*}
\left|\sum_{j=1}^m w_j \frac{(
    \Delta_{n-r^*+1} - \Delta_i)}{m^\beta} K(m(x - x_j))  \right| \leq \lVert K \rVert_{\infty} \frac{\Delta_{n-r^*+1}}{m^{\beta}} \leq \frac{\Delta_{n-r^*+1}}{m^{\beta}} 
\end{align*}
where in the first inequality, we used the fact that for any $x$, there exists at most one $j$ such that  $K(m(x - x_j)) >0$. By definition of $m$ and $r^*$, we have that $\Delta_{n-r^*+1}/m^{\beta} \leq (1/r^*)^{2\beta/(2\beta+1)} \leq 1$, and this is sufficient to prove that the distributions $P_{w,i}$ are non-negative.

We also need to show that we satisfy the condition on the regular drift $\bm{\Delta}_n$. For any $i < n-r^*+1$, we have that $\lVert P_{w,i} - P_{w,i+1}\rVert = 0 \leq \Delta_{i} - \Delta_{i+1}$. If $i \geq n-r^*+1$, we have that
\begin{align*}
\lVert P_{w,i} - P_{w,i+1}\rVert = \frac{\sqrt{\lVert w \rVert_1}}{m^{\beta}} |\Delta_{i} - \Delta_{i+1}| \sqrt{\int K^2(mx) dx} = \frac{\sqrt{\lVert w \rVert_1} \lVert K\rVert}{m^{\beta+{1/2}}}(\Delta_{i} - \Delta_{i+1}) \leq \Delta_{i}-\Delta_{i+1}
\end{align*}
where the last inequality is due to the fact that $\sqrt{\lVert w \rVert_1} \leq m^{1/2}$, $\lVert K \rVert \leq 1$, and $m^{\beta} \geq 1$ (as $m \geq 1$).

Finally, we show that the distributions $P_{w,i}$ are $\beta$-smooth over $[0,1]$. For $i \leq n-r^*+1$, this is trivially true. In order to prove this for $i > n-r^*+1$, we consider
\begin{align*}
P^{(\beta)}_{w,i}(x) = \sum_{j=1}^m w_j (\Delta_{n-r^*+1} - \Delta_i) K^{(\beta)}(m(x-x_j))
\end{align*}
Therefore, by integrating the square of the above expression, we obtain that
\begin{align*}
\int \left(P^{(\beta)}_{w,i}(x)\right)^2 dx &= \lVert w \rVert_1 (\Delta_{n-r^*+1}-\Delta_i)^2 \int K^{(\beta)}(mx)^2 dx \\
&= \lVert w \rVert_1 (\Delta_{n-r^*+1}-\Delta_i)^2  \frac{1}{m} \int K^{(\beta)}(x)^2 dx
\end{align*}
We have that $\lVert w \rVert_1/m \leq 1$, and $(\Delta_{n-r^*+1} - \Delta_i)^2 \leq \Delta_{n-r^*+1}^2 \leq 1$  by assumption of the theorem (as $r^* \geq 1$). Since $\int K^{(\beta)}(x)^2 dx < \infty$ as $K$ is $\beta$-smooth, we can conclude that  $\int \left(P^{(\beta)}_{w,i}(x)\right)^2 dx < \infty$, and $P_{w,i}$ is $\beta$-smooth.
\end{proof}

\begin{proposition}
\label{prop:lb-smooth-distance}
Let $w,w' \in \{0,1\}^m$. We have that
\begin{align*}
\lVert \theta_n(S_w) - \theta_n(S_{w'})\rVert = \sqrt{h(w,w')}\frac{\lVert K\rVert\Delta_{n-r^*+1}}{m^{\beta+1/2}}
\end{align*}
\end{proposition}
\begin{proof}
We have that
\begin{align*}
\lVert \theta_n(S_w) - \theta_n(S_{w'})\rVert = \lVert P_{n,w} - P_{n,w'}\rVert &= \frac{\Delta_{n-r^*+1}}{m^{\beta}}\sqrt{\sum_{j : w_j \neq w'_j} \int K^2(m(x-x_j))dx} \\
&= \sqrt{h(w,w')}\frac{\lVert K\rVert\Delta_{n-r^*+1}}{m^{\beta+1/2}}
\end{align*}
\end{proof}

\begin{proposition}
\label{prop:lb-kl-smooth}
Let $w,w' \in \{0,1\}^m$ such that $h(w,w') =1$. Let $q$ be the bit in which $w$ and $w'$ differ, and assume that $w_q =1$. We have that
\begin{align*}
\KL( S_w \| S_{w'}) = O(1)
\end{align*}
\end{proposition}
\begin{proof}

By using the factorization property of the KL-divergence (\Cref{prop:factorizationKL}), we have that
\begin{align}
\label{eq:tmp-1-smooth}
    \KL(S_{w} \| S_{w'}) &= \sum_{i=1}^n \KL(P_{w,i} \| P_{w',i}) 
 \nonumber \\ &= \sum_{i=n-r^*+1}^n \left\{ \int \left(1+ \frac{(\Delta_{n-r^*+1}-\Delta_i)K(m(x-x_q))}{m^{\beta}}\right)\log\left(1+\frac{(\Delta_{n-r^*+1}-\Delta_i)K(m(x-x_q))}{m^{\beta}}\right) dx \right\}  \enspace ,
\end{align}
where in the last equality we used the definition of KL-divergence and the observation that $P_{w,i}$ and $P_{w',i}$ can only differ on the interval $[(q-1)/m, q/m]$. 
By using the definition of $K(\cdot)$ of \eqref{def-K-lb}, for any $i \geq n-r^*+1$ we can rewrite the above integral as
\begin{align*}
    &\int \left(1+ \frac{(\Delta_{n-r^*+1}-\Delta_i)K(m(x-x_q))}{m^{\beta}}\right)\log\left(1+\frac{(\Delta_{n-r^*+1}-\Delta_i)K(m(x-x_q))}{m^{\beta}}\right) dx     \\
    = &\int \Big[ \left(1+ \frac{(\Delta_{n-r^*+1}-\Delta_i)K_0(4m x)}{m^{\beta}}\right)\log\left(1+\frac{(\Delta_{n-r^*+1}-\Delta_i)K_0(4mx)}{m^{\beta}}\right) \\
    + &\left(1-  \frac{(\Delta_{n-r^*+1}-\Delta_i)K_0(4mx)}{m^{\beta}}\right)\log\left(1-\frac{(\Delta_{n-r^*+1}-\Delta_i)K_0(4mx)}{m^{\beta}}\right)\Big] dx \\
\end{align*}
We can observe that for any $x \in \mathbb{R}$, we have that
\begin{align*}
\left| \frac{\Delta_{n-r^*+1} - \Delta_i}{m^{\beta}} K_0(4mx)\right| \leq \left| \frac{\Delta_{n-r^*+1}}{m^{\beta}}\right|\lVert K \rVert_{\infty} \leq 1
\end{align*}
where in the last inequality, we used $\Delta_{n-r^*+1}/m^{\beta} \leq (1/r^*)^{2\beta/(2\beta+1)} \leq 1$, and the fact that $\lVert K \rVert_{\infty} \leq 1$.
Therefore, we can use \Cref{prop:auxiliary} and show that
\begin{align*}
    \KL(S_{w'} \| S_w) &\leq \frac{2}{m^{2\beta}} \sum_{i=n-r^*+1}^n(\Delta_{n-r^*+1}-\Delta_i)^2 \int K_0(4mx)^2dx \\
    &= O\left( \frac{r^*}{m^{2\beta+1}} \Delta_{n-r^*+1}^2 \lVert K_0 \rVert^2 \right) 
\end{align*}
We have that $\lVert K_0 \rVert^2  \leq \lVert K \rVert^2 \leq 1$. Also, ${m^{2\beta+1}} = r^*$ by definition of $m$, and $\Delta^2_{n-r^*+1} \leq 1$. We can conclude that $\KL(S_{w'} \| S_w) = O(1)$.
\end{proof}
\subsubsection{Proof of the lower bound of \Cref{thm:smooth}}
We can assume that $r^* < n$. In fact, if $r^* = n$, the lower bound immediately follows from the lower bound $\Omega\left(  n^{-2\beta/(2\beta+1)}  \right)$ for learning a $\beta$-smooth density with $n$ independent and identically distributed samples.

Let $\mathcal{S}^* =\{S_{w} : w \in \{0,1\}^m\}$ be defined as at the beginning of  \Cref{sub:lb-smooth}. Due to \Cref{prop:lb-smooth-belong}, we have that
\begin{align*}
\inf_{\hat{\theta}_n} \sup_{S \in \mathcal{S}_n(\bm{\Delta}_n, \beta)}  \Exp_{\bm{X}_n \sim S} \lVert \theta_n(S) - \hat{\theta}_n(
\bm{X}_n)\rVert^2  
\geq \inf_{\hat{\theta}_n} \sup_{S \in \mathcal{S}^*}  \Exp_{\bm{X}_n \sim S} \lVert \theta_n(S) - \hat{\theta}_n(
\bm{X}_n)\rVert^2  
\end{align*}
We can lower bound the right-hand side of the above inequality by using Assouad's Lemma (\Cref{assouad}) with metric $d = L_2$ and $p=2$. By using \Cref{prop:lb-smooth-distance} and~\ref{prop:lb-kl-smooth}, we obtain that 
\begin{align*}
\inf_{\hat{\theta}_n} \sup_{S \in \mathcal{S}^*}  \Exp_{\bm{X}_n \sim S} \lVert \theta_n(S) - \hat{\theta}_n(
\bm{X}_n)\rVert^2 = \Omega\left(  \min_{w \neq w'}\left( 1/\sqrt{h(w,w')} \right) \frac{\Delta_{n-r^*+1}}{m^{\beta-1/2}} \frac{1}{e^{O(1)}} \right)
\end{align*}
We have that $\min_{w \neq w'}\left( 1/\sqrt{h(w,w')} \right) = 1/\sqrt{m}$. By substituting the definition of $m$, we obtain that 
\begin{align*}
\inf_{\hat{\theta}_n} \sup_{S \in \mathcal{S}^*}  \Exp_{\bm{X}_n \sim S} \lVert \theta_n(S) - \hat{\theta}_n(
\bm{X}_n)\rVert^2 = \Omega\left( \Delta_{n-r^*+1} r^{-\beta/(2\beta+1)}\right)
\end{align*}
By assumption, we have that $r^* < n$, hence $\Delta_{n-r^*} = \Omega\left((r^*)^{-\beta/(2\beta+1)}\right)$. Using the definition of $\bm{\Delta}_n$, we have that $\Delta_{n-r^*+1} \geq 
\frac{1}{c}\Delta_{n-r^*} = \Omega\left((r^*)^{-\beta/(2\beta+1)}\right)$. We can conclude that 
\begin{align*}
\inf_{\hat{\theta}_n} \sup_{S \in \mathcal{S}^*}  \Exp_{\bm{X}_n \sim S} \lVert \theta_n(S) - \hat{\theta}_n(
\bm{X}_n)\rVert^2 = \Omega\left(  (r^*)^{-\beta/(2\beta+1)}  (r^*)^{-\beta/(2\beta+1)}\right) = \Omega\left(  (r^*)^{-2\beta/(2\beta+1)}  \right)
\end{align*}
\qed

\end{document}